\documentclass{article}

\usepackage{microtype}
\usepackage{graphicx}
\usepackage{subcaption}
\usepackage{booktabs} 
\usepackage{tabularx}
\usepackage{hyperref}

\usepackage[preprint]{icml2026}

\usepackage{amsmath}
\usepackage{amssymb}
\usepackage{mathtools}
\usepackage{amsthm}

\usepackage[capitalize,noabbrev]{cleveref}

\theoremstyle{plain}
\newtheorem{theorem}{Theorem}[section]
\newtheorem{proposition}[theorem]{Proposition}
\newtheorem{lemma}[theorem]{Lemma}
\newtheorem{corollary}[theorem]{Corollary}
\theoremstyle{definition}
\newtheorem{definition}[theorem]{Definition}
\newtheorem{assumption}[theorem]{Assumption}
\theoremstyle{remark}
\newtheorem{remark}[theorem]{Remark}

\usepackage[textsize=tiny]{todonotes}

\icmltitlerunning{Submission and Formatting Instructions for ICML 2026}

\begin{document}

\twocolumn[
  \icmltitle{GaussDetect-LiNGAM:Causal Direction\\ Identification without Gaussianity test}



  \icmlsetsymbol{equal}{*}

  \begin{icmlauthorlist}
    \icmlauthor{Ziyi Ding}{equal,SIGS}
 \icmlauthor{Xiao-Ping Zhang}{SIGS}
  \end{icmlauthorlist}

  \icmlaffiliation{SIGS}{Tsinghua Shenzhen International Graduate School, Tsinghua University, Shenzhen, China}

  \icmlcorrespondingauthor{Ziyi Ding}{dingzy25@mails.tsinghua.edu.cn}
  \icmlcorrespondingauthor{Xiao-Ping Zhang}{xpzhang@ieee.org}

  \icmlkeywords{Machine Learning, ICML}

  \vskip 0.3in
]



\printAffiliationsAndNotice{}  

\begin{abstract}
We propose GaussDetect-LiNGAM, a novel approach for bivariate causal discovery that eliminates the need for explicit Gaussianity tests by leveraging a fundamental equivalence between noise Gaussianity and residual independence in the reverse regression. Under the standard LiNGAM assumptions of linearity, acyclicity, and exogeneity, we prove that the Gaussianity of the forward-model noise is equivalent to the independence between the regressor and residual in the reverse model. This theoretical insight allows us to replace fragile and sample-sensitive Gaussianity tests with robust kernel-based independence tests. Experimental results validate the equivalence and demonstrate that GaussDetect-LiNGAM maintains high consistency across diverse noise types and sample sizes, while reducing the number of tests per decision (TPD). Our method enhances both the efficiency and practical applicability of causal inference, making LiNGAM more accessible and reliable in real-world scenarios.
\end{abstract}

\section{Introduction}

Causal inference plays a critical role in social sciences and many other 
applied fields, especially when dealing with complex causal relationships.
 Structural Causal Models (SCM) provide a theoretical framework for such 
 inference, capturing the causal relationships between variables through
  explicit structural equation models\cite{10.5555/1642718,spirtesCausationPredictionSearch}.

The Linear Non-Gaussian Acyclic Model (LiNGAM) was introduced to
 address this issue\cite{shimizuLinearNonGaussianAcyclic2006,hyvarinenIndependentComponentAnalysis2000,shenFastKernelICA2007}. By breaking the assumption of Gaussian noise 
 and utilizing the non-Gaussian nature of noise, LiNGAM effectively
  distinguishes causal relationships between variables, allowing for
   a more accurate identification of causal directions.

However, the original formulation of the LiNGAM model relies
 on a constructive algorithm and does not provide strict mathematical
  theorems to prove its identifiability. DirectLiNGAM further addresses 
  the identifiability issue in multivariate cases, but the proofs of these
   theorems strictly follow the assumptions of non-Gaussianity, linearity, acyclicity, 
   and exogeneity\cite{shimizuDirectLiNGAMDirectMethod2011}. Although 
   pairwise-LiNGAM effectively uses the likelihood-ratio approach 
   to determine the causal direction between two variables, it still 
   requires prior knowledge of the four key assumptions: non-Gaussianity,
    linearity, acyclicity, and exogeneity\cite{hyvarinenPairwiseLikelihoodRatios2013}.
   The challenge is that, in many real-world situations,
    noise may be non-Gaussian but cannot be effectively detected, which 
    significantly limits the applicability of the LiNGAM model.

Specifically, this paper focuses on the analysis 
of the bivariate LiNGAM model, as it is often difficult to ensure that all
 variables have non-Gaussian noise in practical applications. Compared to 
 the multivariate case, the bivariate model is more tractable in handling
  local causal relationships and offers better operational feasibility.

This study addresses this issue by proving that,in the bivariate LiNGAM model, 
the Gaussianity test for the noise in the forward model and the independence 
test between the independent variable and noise in the reverse model are 
consistently equivalent. 

A key contribution of this study is the demonstration that,
as long as the bivariate variables satisfy the assumptions of linearity,
  acyclicity, and exogeneity, the non-Gaussian nature of the noise can be 
  directly assessed through the independence test in the reverse model. 
  This approach eliminates the need for prior non-Gaussianity tests,
   allowing researchers to skip the step of checking for non-Gaussian 
   noise distribution. Instead, by focusing solely on the independence 
   test, the process is simplified, making it more efficient. This enables
    the use of various robust and reliable independence testing methods,such as kernel independence testing\cite{akahoKernelMethodCanonical,bachKernelIndependentComponent2003}, 
    significantly improving the practicality and applicability of the LiNGAM 
    model in real-world causal inference tasks.

\section{Notation}

Let \( X \) and \( Y \) be random variables. The joint probability density function of \( X \) and \( Y \) is denoted by \( f_{X,Y}(x, y) \), and the marginal probability density functions of \( X \) and \( Y \) are denoted by \( f_X(x) \) and \( f_Y(y) \), respectively. That is,
$f_{X,Y}(x,y)$ is the joint probability density function of
$X$ and~$Y$,
$f_X(x)$ is the marginal probability density function of~$X$,  
$f_Y(y)$ is the marginal probability density function of~$Y$.

In such situations, if the noise terms are non-Gaussian, we can utilize the LiNGAM model to identify the causal direction between the variables. This is possible because non-Gaussian noise helps distinguish between the direction of causality, even when the relationship between the two variables is complex or bidirectional.

\begin{definition}[Bivariate LiNGAM]
    In the case of two observed variables \( X \) and \( Y \), the Linear Non-Gaussian Acyclic Model (LiNGAM) assumes the following linear relationships with non-Gaussian noise:
    \begin{align}
        Y &= a X + \epsilon_Y\\
        X &= b Y + \epsilon_X
    \end{align}
        
    where:\( X \) and \( Y \) are the observed variables,\( \epsilon_X \) and \( \epsilon_Y \) are the non-Gaussian noise 
        terms for \( X \) and \( Y \), respectively.
    The model assumes that the noise terms \( \epsilon_X \) and 
    \( \epsilon_Y \) are non-Gaussian. 
\end{definition}

In the context of causal inference, two models are often discussed: the forward model and the reverse model. 

Forward model: In this model, we assume that the independent variable \( X \) influences the dependent variable \( Y \), typically expressed as \( Y = a X + \epsilon_Y \), where \( a \) is the regression coefficient, and \( \epsilon_Y \) is the noise term.
  
Reverse model: In the reverse model, we assume that the dependent variable \( Y \) influences the independent variable \( X \), typically expressed as \( X = b Y + \epsilon_X \), where \( b \) is the regression coefficient, and \( \epsilon_X \) is the noise term.

The key challenge is to determine the correct causal direction between the variables, especially when both models seem plausible. The LiNGAM model helps resolve this by leveraging the non-Gaussianity of the noise terms to identify the true direction of causality.

In massive inference and reproducibility studies,
 the same null hypothesis is often examined by more 
 than one procedure—frequentist p-value versus Bayes
  factor, Wald test versus likelihood-ratio test,
   or an original codebase versus its 
   re-implementation. If two such procedures
    always reach identical reject/retain decisions,
     the simpler, faster, 
     or more numerically stable one 
     can be deployed without any risk of 
     contradiction. To justify this substitution 
     rigorously, we first need a sharp, 
     sample-level guarantee that no future 
     data set will ever make the two tests
      disagree.Therefore, Definition~\ref{def2} is formulated as follows:

\begin{definition}[Consistent Equivalence of Hypothesis Tests]
    \label{def2}
    Let there be two tests \( \varphi_1(S) \) and \( \varphi_2(S) \), 
    both used to test the same hypothesis. If the null hypothesis is 
    \( H_0 \) and the alternative hypothesis is \( H_1 \), 
    the test function \( \varphi_i(X) \) (where \( i \in \{1, 2\} \)) 
    indicates whether the null hypothesis \( H_0 \) is rejected based 
    on the sample data \( S \). Specifically:
    \begin{equation}
    \varphi_i(S) = 
    \begin{cases}
        1 & \text{if the null hypothesis is rejected,} \\
        0 & \text{if the null hypothesis is not rejected.}
    \end{cases} 
    \end{equation}
    The two tests \( \varphi_1(S) \) and \( \varphi_2(S) \) 
    are said to be equivalent
     (in the sense of consistent conclusions) if for all \( S \),
    \begin{equation}
    \varphi_1(S) = \varphi_2(S)
\end{equation}
\end{definition}

The definition of the equivalence of hypothesis tests is well-defined because it provides a clear and unambiguous description of two hypothesis tests \( \varphi_1(X) \) and \( \varphi_2(X) \), including their decision rules (reject or fail to reject \( H_0 \)). The equivalence condition requires that the two tests yield the same result for all sample data \( X \), ensuring consistency. The condition is both sufficient and necessary: if the tests are equivalent, they must give the same outcome for every possible \( X \). Furthermore, the definition is mathematically rigorous, using standard notation with no ambiguity, and it is universally applicable to any hypothesis test. Thus, this definition satisfies the criteria for being a well-defined concept.

\section{Problem Formulation}

Consider two competing bivariate linear acyclic models. The goal is to determine the causal direction \textbf{without prior knowledge of whether the noise is non-Gaussian}.
Forward model:$ Y = aX + \epsilon_Y,\; X \perp \epsilon_Y.
$ and Reverse model:$ X = bY + \epsilon_X,\; Y \perp \epsilon_X$.

Two sets of hypothesis tests are defined: one for independence tests (IT) and one for Gaussianity tests (GT):
Independence Tests (IT):
\begin{equation}
    \label{eq1}
    H_{IT_{10}}: Y \perp \epsilon_X, \;H_{IT_{11}}: Y \not\perp \epsilon_X 
\end{equation}
\begin{equation}
    \label{eq2}
    H_{IT_{20}}: X \perp \epsilon_Y, \;H_{IT_{21}}: X \not\perp \epsilon_Y
\end{equation}
Gaussianity Tests (GT):
\begin{equation}
    \label{eq3}
H_{GT_{10}}: \epsilon_Y \sim N(0, \sigma_Y^2), \; H_{GT_{11}}: \epsilon_Y \not\sim \mathcal{N}(0,\sigma_Y^{2})
\end{equation}
\begin{equation}
    \label{eq4}
    H_{GT_{20}}: \epsilon_X \sim N(0, \sigma_X^2), \; H_{GT_{21}}: \epsilon_X \not\sim \mathcal{N}(0,\sigma_X^{2})
    \end{equation}
Here, $\varphi_{IT_1}, \varphi_{IT_2}$ are the decision variables for the independence tests, and $\varphi_{GT_1}, \varphi_{GT_2}$ are the decision variables for the Gaussianity tests, both taking values from $\{0, 1\}$ (where $1$ means do not reject, and $0$ means reject).

If the data generation is assumed to follow the forward model, then $H_{IT_{10}}$ and $H_{GT_{10}}$ should be tested. On the other hand, if the data generation follows the reversed model, then $H_{IT_{20}}$ and $H_{GT_{20}}$ should be tested.

Next, the expected equivalence will be discussed. Under the concept of consistent equivalence, it is expected that:
\begin{equation}
    \varphi_{IT_1} = \varphi_{GT_1}
\end{equation}
\begin{equation}
    \varphi_{IT_2} = \varphi_{GT_2}
\end{equation}

This issue will be explored further, with the expectation that within this model class, independence tests (IT) and Gaussianity tests (GT) will be considered equivalent at the sample level. If this equivalence holds, independence tests can be used to replace Gaussianity tests.

Based on the equivalence, the direction identification via independence tests is considered. The question to be explored is whether a criterion or algorithm can be developed to determine the direction (forward or reversed) under the assumptions of linearity, no cycles, and exogeneity. If the direction cannot be determined, the issue is how the algorithm should handle cases of "undetermined" or "inconclusive."

It is worth mentioning that this does not remove the role of non-Gaussianity for identifiability: if the noise is Gaussian, the direction is not identifiable. Rather, it removes the need to \emph{know in advance} whether the noise is non-Gaussian; the independence test outcomes themselves indicate whether one is in the identifiable (non-Gaussian) or non-identifiable (Gaussian) regime.

\section{Consistent Equivalence of Gaussianity Test and Independence Test in the LiNGAM}

This section focuses on the relationship between the Gaussianity 
test and the independence test in the LiNGAM model. 
First, we present Lemma 1
\cite{darmoisAnalyseGeneraleLiaisons1953,lukacsPropertyNormalDistribution1954}:

\begin{lemma}[Skitovich--Darmois Theorem]
    Let $X_1,\allowbreak X_2,\allowbreak \ldots,\allowbreak X_n$ be independent random variables.  
    If there exist non-zero constants $c_1, c_2, \dots, c_n$, $ d_1, d_2, \dots, d_n$ such that the linear combinations  
    $Y_1 = \sum_{i=1}^{n} c_i X_i$ and $Y_2 = \sum_{j=1}^{n} d_j X_j$ are independent,  
    then $X_1, X_2, \ldots, X_n$ are normally distributed.
    \label{lemma1}
    \end{lemma}

The conclusion of Lemma~\ref{lemma1} is non-trivial. Generally, we can infer independence from the properties of the distribution, but Lemma~\ref{lemma1} (the Skitovich-Darmois theorem) allows us to infer the distributional properties of random variables based on their independence.

\begin{theorem}
    \label{thm1}
    If \( X \) and \( Y \) satisfy the forward model \( Y = aX + \epsilon_Y \) and the reverse model \( X = bY + \epsilon_X \), a sufficient and necessary condition for \( \binom{Y}{\epsilon_X} \) to be independent is that \( \binom{X}{\epsilon_Y} \) is independent and both follow a normal distribution.
\end{theorem}

\begin{proof}
    \textbf{Sufficiency:}  
    Assume that the components of \( \binom{X}{\epsilon_Y} \), namely \( X \) and \( \epsilon_Y \), are independent, and both \( X \) and \( \epsilon_Y \) follow normal distributions.  
    Furthermore, we have the following relationship between \( \binom{Y}{\epsilon_X} \) and \( \binom{X}{\epsilon_Y} \):

    \[
    \binom{Y}{\epsilon_X} = \left( \begin{matrix} a & 1 \\ 1 - ba & -b \end{matrix} \right) \binom{X}{\epsilon_Y}
    \]

    Since \( \binom{X}{\epsilon_Y} \) consists of independent components, and both components are normally distributed, it follows that \( \binom{Y}{\epsilon_X} \) also follows a normal distribution.  

    Next, consider the covariance matrix \( \text{Cov}(Y, \epsilon_X) \):

    \[
    \text{Cov}(Y, \epsilon_X) = \left( \begin{matrix} a & 1 \\ 1 - ba & -b \end{matrix} \right) \text{Cov}(X, \epsilon_Y) \left( \begin{matrix} a & 1 \\ 1 - ba & -b \end{matrix} \right)^T
    \]

    For simplicity, assume that the components of \( \binom{X}{\epsilon_Y} \) have unit variances (if not, we can standardize the variables). Thus, we have:

    \[
    \text{Cov}(X, \epsilon_Y) = \left( \begin{matrix} 1 & 0 \\ 0 & 1 \end{matrix} \right)
    \]

    Since \( \left( \begin{matrix} a & 1 \\ 1 - ba & -b \end{matrix} \right) \) is an orthogonal matrix, we know that:

    \[
    \left( \begin{matrix} a & 1 \\ 1 - ba & -b \end{matrix} \right) \left( \begin{matrix} a & 1 \\ 1 - ba & -b \end{matrix} \right)^T = \left( \begin{matrix} 1 & 0 \\ 0 & 1 \end{matrix} \right)
    \]

    Thus, we conclude that:

    \[
    \text{Cov}(Y, \epsilon_X) = \left( \begin{matrix} 1 & 0 \\ 0 & 1 \end{matrix} \right)
    \]

    Since \( \binom{Y}{\epsilon_X} \) is a multivariate normal distribution and its components are uncorrelated, they must also be independent. Therefore, we have shown that the components of \( \binom{Y}{\epsilon_X} \) are independent.

    \textbf{Necessity:}  
    Since the components of \( \binom{Y}{\epsilon_X} \), namely \( Y \) and \( \epsilon_X \), are independent, we now look at the relationship between \( \binom{X}{\epsilon_Y} \) and \( \binom{Y}{\epsilon_X} \):

    \[
    \binom{X}{\epsilon_Y} = \left( \begin{matrix} a & 1 - ba \\ 1 & -b \end{matrix} \right) \binom{Y}{\epsilon_X}
    \]

    Since \( \binom{Y}{\epsilon_X} \) is independent and the linear combination \( \binom{X}{\epsilon_Y} \) is based on it, by the Skitovich-Darmois theorem, we conclude that \( \binom{Y}{\epsilon_X} \) must follow a normal distribution.

    Furthermore, by using the same reasoning as in the sufficiency proof, we can show that the components of \( \binom{Y}{\epsilon_X} \) are also independent.

    Thus, the proof is complete. 
    \end{proof}

Theorem~\ref{thm1} specifies the conditions under which causal direction can and cannot be identified, and rigorously proves the identifiability of the bivariate LiNGAM model.
Specifically, if \( Y \perp \epsilon_X \), then \( \epsilon_Y \) 
must follow a normal distribution. 
In the forward model, if the noise does not follow a normal distribution, 
the causal direction can be identified. However, 
if the noise satisfies the assumption of Gaussianity,
 the causal direction cannot be determined.

From a theoretical perspective, as long as the noise does not follow a normal distribution under certain conditions, the LiNGAM model can identify the causal direction. However, in practice, the only data available is usually observational data, and it is impossible to know in advance whether the noise follows a normal distribution. Therefore, hypothesis testing must be used. This leads to the following Theorem~\ref{thm2}.

\begin{theorem}
    \label{thm2}
    In the causal direction identification problem of the Bivariate LiNGAM , 
    let \( \varphi_{IT_i}(S) \) ,\(i \in \{1,2\}\)be the independence test mentioned in \eqref{eq1} and \eqref{eq2}, and \( \varphi_{GT_i}(S) \) ,\(i \in \{1,2\}\)
    be the Gaussianity test mentioned in \eqref{eq3} and \eqref{eq4}.
     Then:

     \begin{equation}
        \varphi_{IT_i}(S) = \varphi_{GT_i}(S), \quad \forall S, \; i \in \{1,2\}
    \end{equation}

    This means that both tests yield consistent conclusions for all sample data.
\end{theorem}

\begin{proof}
    We first show \( \varphi_{IT_1}(S) = \varphi_{GT_1}(S) \).
    When \( \varphi_{IT_1}(S) = 1 \), it means that \( Y \) and \( \epsilon_X \) are independent. 
    According to Theorem 1 (with \(X\) and \(Y\) interchanged), if \( Y \) and \( \epsilon_X \) are independent, 
    then \( \epsilon_Y \) must follow a normal distribution. Therefore, the Gaussianity test 
    \( \varphi_{GT_1}(S) \) will conclude that \( \epsilon_Y \sim \mathcal{N}(0,\sigma_Y^2) \), 
    i.e., \( \varphi_{GT_1}(S) = 1 \). Thus, when \( \varphi_{IT_1}(S) = 1 \), we also have 
    \( \varphi_{GT_1}(S) = 1 \), and both tests give consistent conclusions.

    When \( \varphi_{IT_1}(S) = 0 \), it means that \( Y \) and \( \epsilon_X \) are not independent, 
    indicating a dependency between them. In this case, \( \epsilon_Y \) cannot follow a normal distribution, 
    because if \( \epsilon_Y \) were Gaussian, it would be independent of \( Y \). Therefore, the 
    Gaussianity test \( \varphi_{GT_1}(S) \) will conclude that \( \epsilon_Y \) does not follow a normal 
    distribution, i.e., \( \varphi_{GT_1}(S) = 0 \). Hence, when \( \varphi_{IT_1}(S) = 0 \), we also have 
    \( \varphi_{GT_1}(S) = 0 \), and both tests yield consistent conclusions.

    For \( i=2 \), the same argument applies after interchanging \( X \) and \( Y \), which shows 
    \( \varphi_{IT_2}(S) = \varphi_{GT_2}(S) \). Combining both cases yields the claim.
\end{proof}

Gaussianity tests often lack accuracy and power: 
they are sample-size sensitive and falter under 
misspecification. Independence tests are more mature
 and robust; Theorem~\ref{thm2} proves their equivalence in 
 bivariate LiNGAM, so the latter can replace the 
 former when Gaussianity checks underperform.

Specifically, Shapiro–Wilk (SW) excels with small samples but over-rejects in large ones and weakens under leptokurtic distributions; Anderson–Darling (AD) shares the small-sample merit yet loses accuracy and speed as data grow; Jarque–Bera (JB), though universally applicable, becomes unstable with few observations and loses power when tails alone deviate. These traditional methods hinge on specific sample sizes and distributional forms, rendering them unreliable in complex scenarios, whereas independence tests remain flexible across varying data conditions.

By Theorem~\ref{thm2}, for each $i\in\{1,2\}$ the independence and Gaussianity tests are equivalent,
$\varphi_{IT_i}(S)=\varphi_{GT_i}(S)$. In particular, let
\[
\varphi_1 =
\begin{cases}
1, & \text{if } Y \perp \epsilon_X \ \ (\text{equivalently } \epsilon_Y \sim \mathcal{N}(0,\sigma_Y^2)),\\
0, & \text{otherwise,}
\end{cases}
\]
\[
\varphi_2 =
\begin{cases}
1, & \text{if } X \perp \epsilon_Y \ \ (\text{equivalently } \epsilon_X \sim \mathcal{N}(0,\sigma_X^2)),\\
0, & \text{otherwise.}
\end{cases}
\]
Hence, the joint decision $(\varphi_1,\varphi_2)$ can be interpreted as follows:
\begin{itemize}
\item \textbf{$(\varphi_1,\varphi_2)=(1,0)$ $\Rightarrow$ $Y\to X$.}
Accepting $Y \perp \epsilon_X$ (reverse regression residual independent of the predictor) and
rejecting $X \perp \epsilon_Y$ (forward regression residual dependent on the predictor)
is exactly the LiNGAM exogeneity pattern for the causal direction $Y\to X$.

\item \textbf{$(\varphi_1,\varphi_2)=(0,1)$ $\Rightarrow$ $X\to Y$.}
Symmetric to the previous case: independence holds only in the forward regression, indicating $X\to Y$.

\item \textbf{$(\varphi_1,\varphi_2)=(1,1)$ $\Rightarrow$ Undetermined (noise effectively Gaussian).}
Both independence tests are accepted, and by Theorem~\ref{thm2} both Gaussianity tests are also accepted,
so $\epsilon_X,\epsilon_Y$ behave as Gaussian. In the linear-Gaussian setting the direction is not identifiable,
thus we abstain.

\item \textbf{$(\varphi_1,\varphi_2)=(0,0)$ $\Rightarrow$ Inconclusive.}
Both independence tests are rejected (and, equivalently, both Gaussianity tests are rejected).
This offers insufficient evidence for either direction and may indicate finite-sample power issues,
model misspecification (e.g., nonlinearity, hidden confounding), or near-Gaussian disturbances.
\end{itemize}

\section{Algorithms}

Based on the conclusion in Theorem~\ref{thm2}, the GaussDetect-LiNGAM algorithm is proposed. 

By Theorem~\ref{thm2}, the independence tests and Gaussianity tests are equivalent decision rules for the bivariate case. 
This allows us to release the a priori knowledge requirement on the noise non-Gaussianity: 
GaussDetect-LiNGAM does not assume that the disturbance terms are known to be non-Gaussian in advance; 
instead, it diagnoses (non-)Gaussianity from the data via the tests and proceeds accordingly.
Importantly, this does not mean that non-Gaussian noise is unnecessary for identifiability; 
it only means we do not need to know it a priori. When both residuals are Gaussian (both tests accepted), the algorithm abstains and declares LiNGAM inapplicable.

In contrast, classical LiNGAM methods require the full set of assumptions to hold \emph{ex ante} and simultaneously:
(i) linearity, (ii) acyclicity, (iii) exogeneity (no hidden confounding), and (iv) non-Gaussian disturbances. 
All four are indispensable for identifiability in the traditional framework. 
GaussDetect-LiNGAM keeps (i)–(iii) as standing assumptions and replaces the \emph{a priori} imposition of (iv) 
with a data-driven diagnostic, as summarized in Algorithm~\ref{alg:1}.
\begin{algorithm}[tb]
  \caption{GaussDetect-LiNGAM}
  \label{alg:1}
  \begin{algorithmic}
    \STATE {\bfseries Input:} Dataset $\mathbf{X} = \{X, Y\}$ (two continuous variables)
    \STATE {\bfseries Output:} Causal direction: $X \to Y$, $Y \to X$, ``Gaussian noise'', or ``Inconclusive''
    \STATE
    \STATE {\bfseries Step 1:} Standardize variables $X$ and $Y$
    \STATE
    \STATE {\bfseries Step 2:} Fit regression models
    \STATE \hspace{0.5cm} Forward: $Y = aX + \epsilon_Y$
    \STATE \hspace{0.5cm} Reversed: $X = bY + \epsilon_X$
    \STATE
    \STATE {\bfseries Step 3:} Perform independence tests
    \STATE \hspace{0.5cm} $H_{10}$: Test $X \perp \epsilon_Y$
    \STATE \hspace{0.5cm} $H_{20}$: Test $Y \perp \epsilon_X$
    \STATE
    \STATE {\bfseries Step 4:} Determine causal direction
    \IF{$H_{10}$ accepted and $H_{20}$ rejected}
        \STATE $X \to Y$
    \ELSIF{$H_{20}$ accepted and $H_{10}$ rejected}
        \STATE $Y \to X$
    \ELSIF{$H_{10}$ and $H_{20}$ both accepted}
        \STATE ``Gaussian noise'' (LiNGAM not applicable)
    \ELSE
        \STATE ``Inconclusive''
    \ENDIF
    \STATE
    \STATE {\bfseries return} causal direction
  \end{algorithmic}
\end{algorithm}
\begin{table}[t]
  \caption{LiNGAM Variants Comparison}
  \label{tab:lingam-comparison}
  \begin{center}
    \begin{small}
      \begin{tabularx}{\linewidth}{l>{\centering\arraybackslash}X>{\centering\arraybackslash}X>{\centering\arraybackslash}X>{\centering\arraybackslash}X}
        \toprule
        \textbf{Variant} & \textbf{L, A, E} & \textbf{NG} & \textbf{RoV} & \textbf{Return} \\
        \midrule
        ICA         & Must    & Must      & Global    & DAG  \\
        Direct      & Must    & Must      & Global    & DAG  \\
        Pairwise    & Must    & Must      & Local     & DAG  \\
        GaussDetect & Must    & Optional  & Local     & DAG,ND \\
        \bottomrule
      \end{tabularx}
    \end{small}
  \end{center}
  \vskip -0.1in
  \begin{flushleft}
    \small\textit{Note:} All methods are LiNGAM variants. L = Linearity, A = Acyclicity, E = Exogeneity, NG = Non-Gaussianity, RoV = Range of Variables (Global/Local), DAG = Directed Acyclic Graph, ND = Noise Distribution. Must = All three assumptions (L, A, E) must be satisfied.
  \end{flushleft}
\end{table}

The advantage of the GaussDetect-LiNGAM algorithm lies in its ability to automatically detect the Gaussian nature of the noise during causal inference, without requiring a priori assumptions about the noise distribution. By performing independence tests, GaussDetect-LiNGAM can handle data with varying noise types, avoiding the dependency on non-Gaussian noise assumptions typical of traditional LiNGAM algorithms. This feature makes GaussDetect-LiNGAM more adaptable in scenarios with unknown or changing noise, providing more accurate and reliable causal inference results.
\begin{figure}[ht]
  \vskip 0.2in
  \begin{center}
    \centerline{\includegraphics[width=1.1\columnwidth]{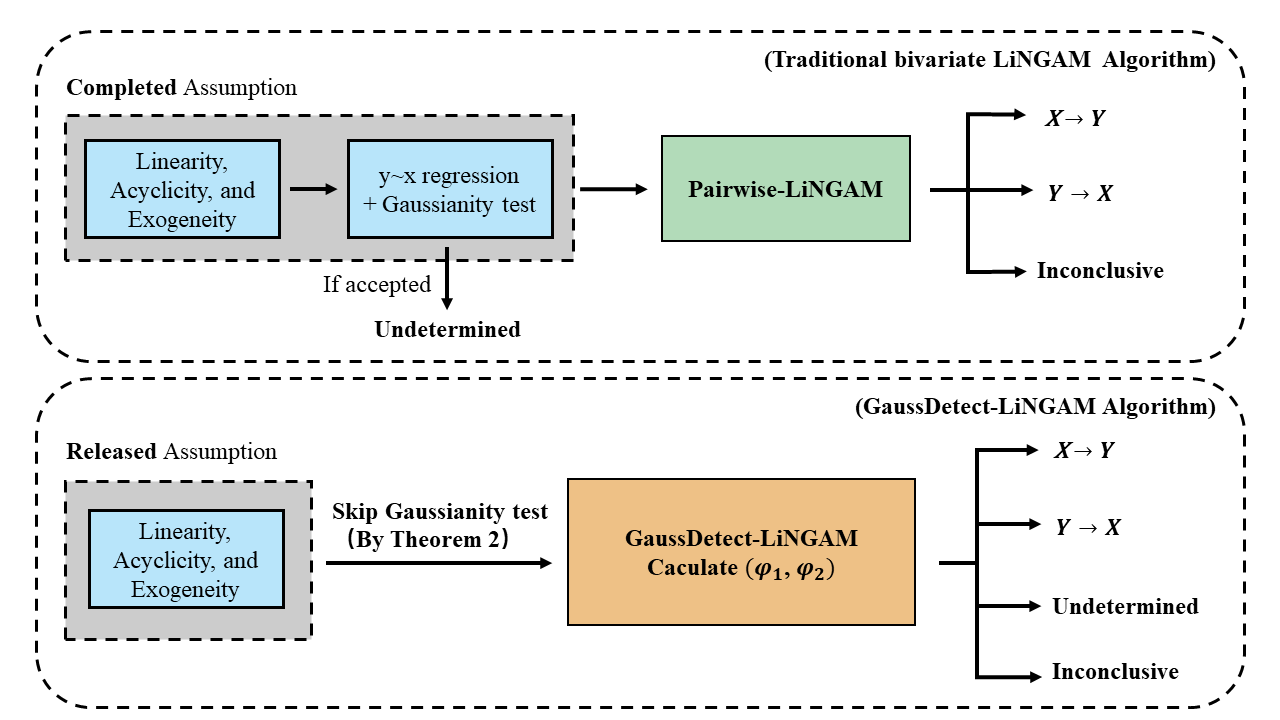}}
    \caption{
      This figure compares the Traditional Pairwise LiNGAM and GaussDetect-LiNGAM algorithms. Pairwise LiNGAM requires completed assumptions (linearity, acyclicity, exogeneity, and non-Gaussianity) and involves two steps: a Gaussianity test followed by an independence test. In contrast, GaussDetect-LiNGAM only requires the released assumptions (linearity, acyclicity, and exogeneity) and can infer causality with just one independence test, without needing the Gaussianity test.
    }
    \label{fig:AlgorithmComparison}
  \end{center}
\end{figure}

\section{Experiments}

\subsection{Experimental Setup}

Experiment 1 aims to validate the correctness of Theorem~\ref{thm2}
 by evaluating its accuracy through the consistency rate
  (Consistency Rate). Experiment 2 aims to compare 
  GaussDetect-LiNGAM and Pairwise-LiNGAM across different noise types (Gaussian and Non-Gaussian) and varying sample sizes (400, 800, 1600).
   The core goal of this experiment is to evaluate the differences
   between these two algorithms in terms of tests per decision (TPD).

\subsubsection{Sample Generation}

We generated two sets of samples:
1.Gaussian Noise: Samples drawn from the standard Gaussian distribution.
2.Non-Gaussian Noise: Including Exponential, Laplace, and Poisson distributions. Each of these non-Gaussian noise types generated the same number of samples as the Gaussian noise.The generated datasets (\(X\) and \(Y\)) were divided into two parts:
1.\(X\) is a random sample generated from the standard normal distribution.
2.\(Y\) is generated by the formula \(Y = 2X +  \epsilon_{Y}\), where the \(\text{noise}\) is composed of either Gaussian or non-Gaussian noise.

Each set of samples (Gaussian and non-Gaussian noise) was 
divided into multiple small batches. Each batch underwent Gaussianity tests and independence tests, and we calculated the tests per decision (TPD) and consistency rate (Consistency Rate).

\subsection{Consistency Rate Experiment}

The consistency rate (Consistency Rate) measures whether the results of the Gaussianity test and independence test in each batch are consistent. This experiment aims to evaluate whether GaussDetect-LiNGAM can consistently determine the validity of the data under different noise types and sample sizes.

\subsubsection{Experiment Procedure}

1.For each dataset, the data is divided into 
    multiple small batches.
    
2.First, a Gaussianity test is performed.

3.For samples that pass the Gaussianity test, we perform regression on the residual noise and perform the independence test between the residual noise and the predictor variable.

4.Record whether each batch yields consistent conclusions (whether the Gaussianity test and independence test results are consistent).

5.Calculate the consistency rate for each noise type and sample size.

\subsubsection{Results}

\begin{figure}[ht]
  \vskip 0.2in
  \begin{center}
    \centerline{\includegraphics[width=\columnwidth]{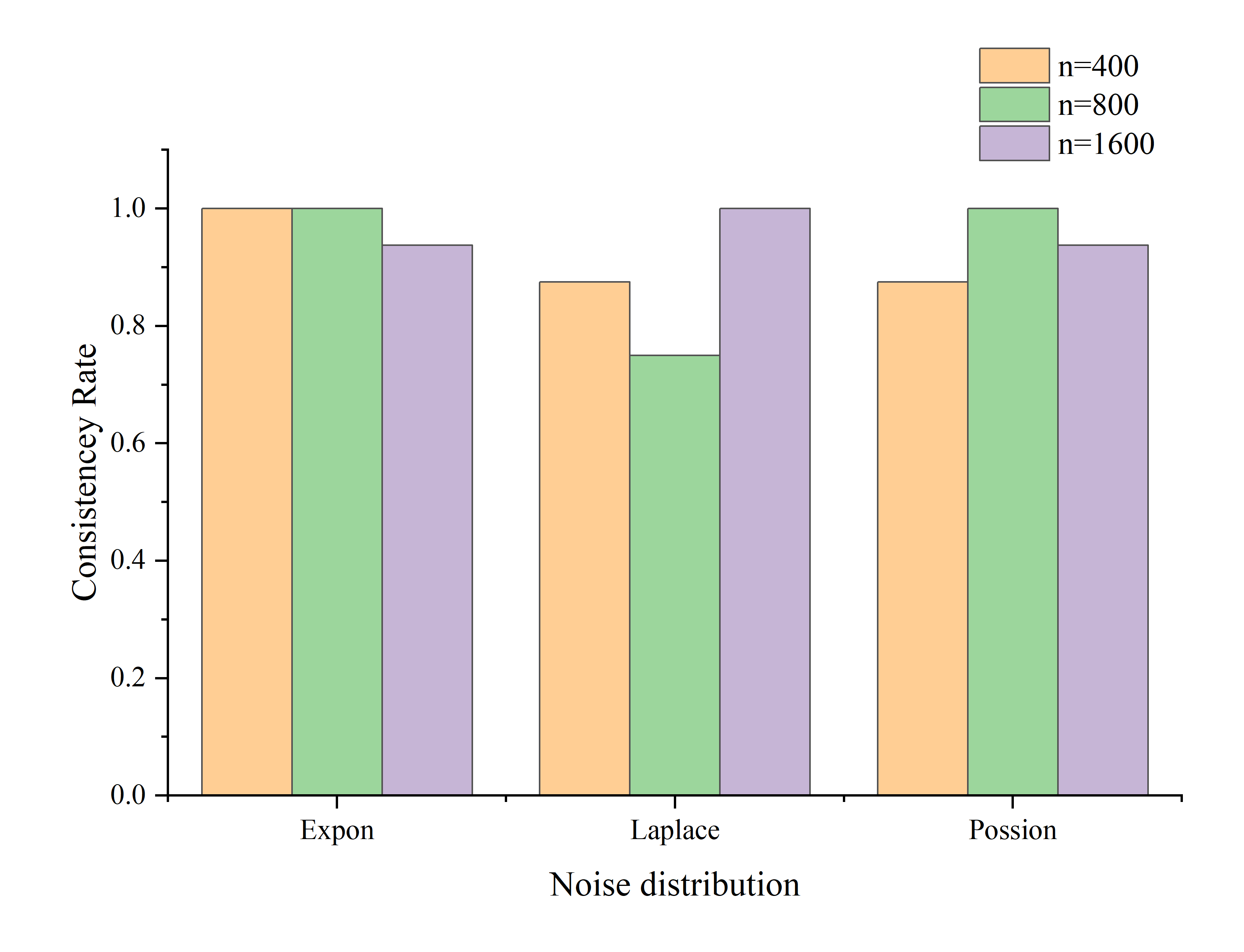}}
    \caption{
      The figure illustrates the Consistency Rate across different noise distributions (Exponential, Laplace, and Poisson) and sample sizes (n=400, n=800, n=1600). The bars in different colors represent the performance of the consistency rate for each sample size.
    }
    \label{fig:ConsistenceyRate}
  \end{center}
\end{figure}

As shown in the figure~\ref{fig:ConsistenceyRate}, GaussDetect-LiNGAM consistently demonstrates a high consistency rate, especially in larger sample sizes, with many cases achieving a consistency rate of 1. This indicates that GaussDetect-LiNGAM is effective at providing consistent judgments across different noise types.

In the case of non-Gaussian noise (such as Laplace and Poisson distributions), we also observe fairly consistent results. Although the Gaussianity test does not always pass, when the noise type is close to Gaussian, GaussDetect-LiNGAM still achieves a high consistency rate.

\subsection{Tests per Decision (TPD) Experiment}

TPD(Tests per Decision) refers to the average number of tests required by GaussDetect-LiNGAM and Pairwise-LiNGAM to make a decision. This experiment compares the decision-making efficiency of both algorithms across different sample sizes and noise types.

\subsubsection{Experiment Procedure}

1. For each dataset, the data is divided into multiple small batches.

2.For GaussDetect-LiNGAM, only independence tests are performed, and the number of tests per decision is calculated.

3.For Pairwise-LiNGAM, we first perform a Gaussianity test, and if the Gaussianity test passes, we proceed with two independence tests.

4.Calculate TPD (tests per decision) and compute the average number of tests for each sample size and noise type.

\subsubsection{Results}

\begin{figure}[ht]
  \vskip 0.2in
  \begin{center}
    \centerline{\includegraphics[width=\columnwidth]{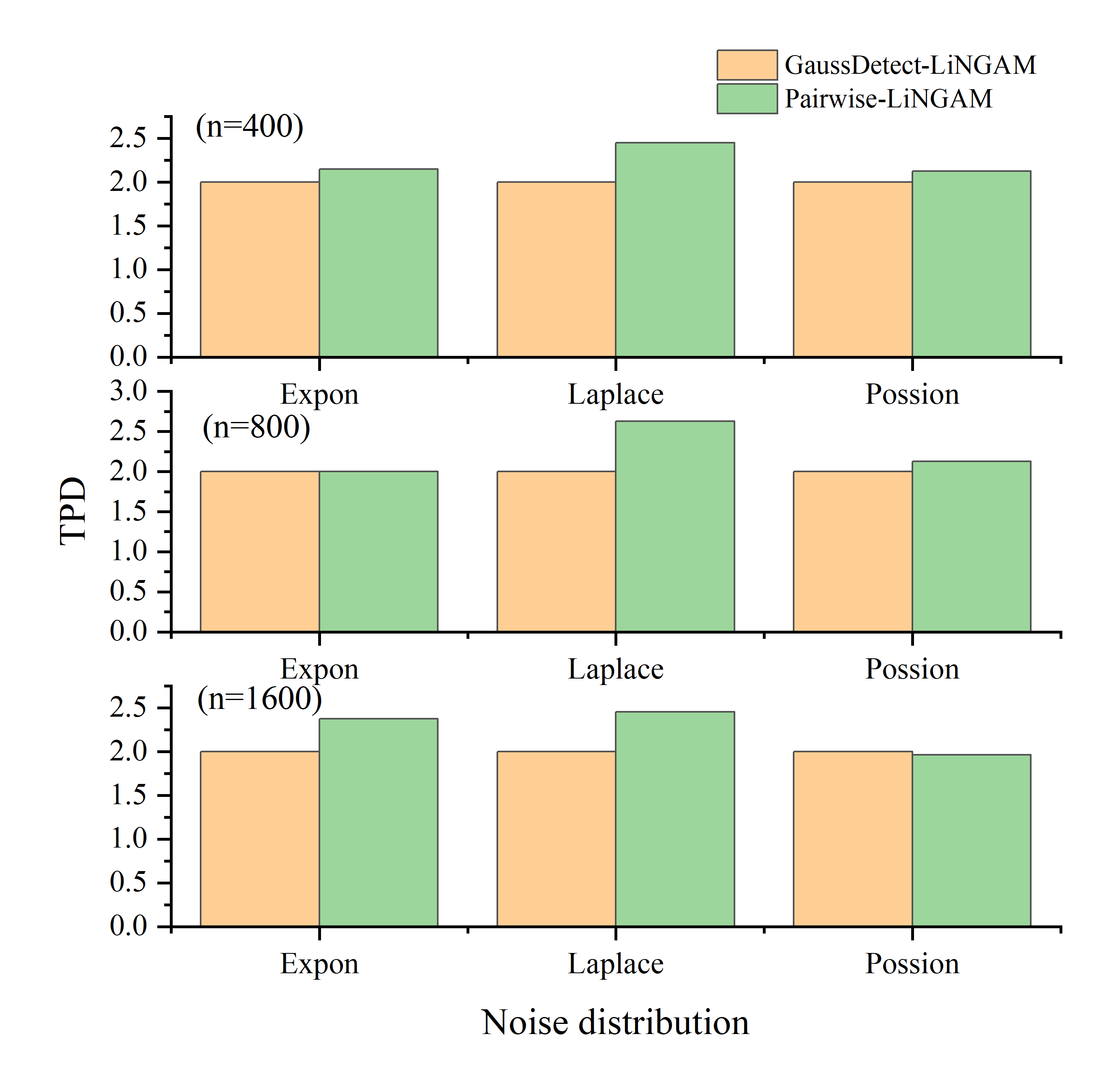}}
    \caption{
      This figure illustrates the Tests per Decision (TPD) required by GaussDetect-LiNGAM and Pairwise-LiNGAM across different noise distributions (Exponential, Laplace, and Poisson) and sample sizes (n=400, n=800, n=1600). The bars in different colors represent the performance of GaussDetect-LiNGAM (in orange) and Pairwise-LiNGAM (in green).
    }
    \label{fig:TPD}
  \end{center}
\end{figure}

As shown in the figure~\ref{fig:TPD}, GaussDetect-LiNGAM requires only one independence test per decision, so its TPD is generally low, around 2 tests. In contrast, Pairwise-LiNGAM requires first performing the Gaussianity test, followed by two independence tests, which results in a higher TPD, typically greater than 2 tests. This suggests that GaussDetect-LiNGAM is more efficient in terms of the number of tests required per decision, while Pairwise-LiNGAM has a higher computational cost due to the additional Gaussianity test.

\subsection{Experiment Summary}

Consistency Rate: The results demonstrate the correctness of
 Theorem~\ref{thm2} and the accuracy of GaussDetect-LiNGAM. 
 It consistently achieves a high consistency rate across all 
 noise types, particularly in cases involving non-Gaussian noise 
 such as Laplace and Poisson distributions. This validates the 
 effectiveness of GaussDetect-LiNGAM in handling non-Gaussian 
 noise while maintaining reliable causality inference.
  
Tests per Decision (TPD): GaussDetect-LiNGAM requires only independence tests per decision, which results in a lower TPD of around 2. On the other hand, Pairwise-LiNGAM requires an additional Gaussianity test and two independence tests, resulting in a higher TPD, generally greater than 2.

These results indicate that GaussDetect-LiNGAM is more efficient, especially when handling non-Gaussian noise, with a lower computational cost, while Pairwise-LiNGAM provides more detailed analysis but requires more tests per decision and incurs higher computational overhead.

\section{Conclusion}
We have shown that, in the bivariate LiNGAM setting, the Gaussianity of noise can be equivalently assessed through an independence test in the reverse model. This theoretical result eliminates the need for explicit non-Gaussianity assumptions and enables the use of more flexible and powerful independence testing methods. The proposed GaussDetect-LiNGAM algorithm simplifies causal direction identification while maintaining robust performance under diverse data conditions. These contributions make LiNGAM more accessible and reliable for practical causal inference applications.

\bibliography{example_paper}

@book{10.5555/1642718,
author = {Pearl, Judea},
title = {Causality: Models, Reasoning and Inference},
year = {2009},
isbn = {052189560X},
publisher = {Cambridge University Press},
address = {USA},
edition = {2nd}
}

@article{spirtesCausationPredictionSearch,
  title = {Causation, {{Prediction}}, and {{Search}}},
  author = {Spirtes, Peter and Glymour, Clark and Scheines, Richard},
  langid = {english}
}

@article{hyvarinenIndependentComponentAnalysis2000,
  title = {Independent Component Analysis: Algorithms and Applications},
  shorttitle = {Independent Component Analysis},
  author = {Hyvärinen, A. and Oja, E.},
  date = {2000},
  journaltitle = {Neural Networks: The Official Journal of the International Neural Network Society},
  shortjournal = {Neural Netw},
  volume = {13},
  number = {4--5},
  eprint = {10946390},
  eprinttype = {pubmed},
  pages = {411--430},
  issn = {0893-6080},
  doi = {10.1016/s0893-6080(00)00026-5},
  langid = {english}
}

@online{shimizuDirectLiNGAMDirectMethod2011,
  title = {{{DirectLiNGAM}}: {{A}} Direct Method for Learning a Linear Non-{{Gaussian}} Structural Equation Model},
  shorttitle = {{{DirectLiNGAM}}},
  author = {Shimizu, Shohei and Inazumi, Takanori and Sogawa, Yasuhiro and Hyvarinen, Aapo and Kawahara, Yoshinobu and Washio, Takashi and Hoyer, Patrik O. and Bollen, Kenneth},
  date = {2011-04-07},
  eprint = {1101.2489},
  eprinttype = {arXiv},
  eprintclass = {stat},
  doi = {10.48550/arXiv.1101.2489},
  url = {http://arxiv.org/abs/1101.2489},
  urldate = {2025-10-09},
  pubstate = {prepublished}
}

@article{shimizuLinearNonGaussianAcyclic2006,
  title = {A Linear Non-{{Gaussian}} Acyclic Model for Causal Discovery.},
  author = {Shimizu, Shohei and Hoyer, Patrik O. and Hyvärinen, Aapo and Kerminen, Antti and Jordan, Michael},
  date = {2006},
  journaltitle = {Journal of Machine Learning Research},
  volume = {7},
  number = {10},
  url = {http://www.jmlr.org/papers/volume7/shimizu06a/shimizu06a.pdf},
  urldate = {2025-10-08}
}

@article{hyvarinenPairwiseLikelihoodRatios2013,
  title = {Pairwise {{Likelihood Ratios}} for {{Estimation}} of {{Non-Gaussian Structural Equation Models}}},
  author = {Hyvärinen, Aapo and Smith, Stephen M.},
  date = {2013-01},
  journaltitle = {Journal of machine learning research: JMLR},
  shortjournal = {J Mach Learn Res},
  volume = {14},
  eprint = {31695580},
  eprinttype = {pubmed},
  pages = {111--152},
  issn = {1532-4435},
  issue = {Jan},
  langid = {english},
  pmcid = {PMC6834441}
}

@inproceedings{shenFastKernelICA2007,
  title = {Fast {{Kernel ICA}} Using an {{Approximate Newton Method}}},
  booktitle = {Proceedings of the {{Eleventh International Conference}} on {{Artificial Intelligence}} and {{Statistics}}},
  author = {Shen, Hao and Jegelka, Stefanie and Gretton, Arthur},
  date = {2007-03-11},
  pages = {476--483},
  publisher = {PMLR},
  issn = {1938-7228},
  url = {https://proceedings.mlr.press/v2/shen07a.html},
  urldate = {2025-10-10},
  eventtitle = {Artificial {{Intelligence}} and {{Statistics}}},
  langid = {english}
}

@article{akahoKernelMethodCanonical,
  title = {A Kernel Method for Canonical Correlation Analysis},
  author = {Akaho, Shotaro},
  langid = {english}
}

@article{bachKernelIndependentComponent2003,
  title = {Kernel Independent Component Analysis},
  author = {Bach, Francis R. and Jordan, Michael I.},
  date = {2003-03-01},
  journaltitle = {J. Mach. Learn. Res.},
  volume = {3},
  pages = {1--48},
  issn = {1532-4435},
  doi = {10.1162/153244303768966085},
  url = {https://dl.acm.org/doi/10.1162/153244303768966085},
  urldate = {2025-10-10},
  issue = {null}
}

@article{darmoisAnalyseGeneraleLiaisons1953,
  title = {Analyse Générale Des Liaisons Stochastiques: Etude Particulière de l'analyse Factorielle Linéaire},
  shorttitle = {Analyse Générale Des Liaisons Stochastiques},
  author = {Darmois, G.},
  date = {1953},
  journaltitle = {Revue de l'Institut International de Statistique / Review of the International Statistical Institute},
  volume = {21},
  number = {1/2},
  eprint = {1401511},
  eprinttype = {jstor},
  pages = {2--8},
  publisher = {[International Statistical Institute (ISI), Wiley]},
  issn = {0373-1138},
  doi = {10.2307/1401511},
  url = {https://www.jstor.org/stable/1401511},
  urldate = {2025-10-09}
}

@article{lukacsPropertyNormalDistribution1954,
  title = {A {{Property}} of the {{Normal Distribution}}},
  author = {Lukacs, Eugene and King, Edgar P.},
  date = {1954-06},
  journaltitle = {The Annals of Mathematical Statistics},
  volume = {25},
  number = {2},
  pages = {389--394},
  publisher = {Institute of Mathematical Statistics},
  issn = {0003-4851, 2168-8990},
  doi = {10.1214/aoms/1177728796},
  url = {https://projecteuclid.org/journals/annals-of-mathematical-statistics/volume-25/issue-2/A-Property-of-the-Normal-Distribution/10.1214/aoms/1177728796.full},
  urldate = {2025-10-09}
}
\bibliographystyle{icml2026}

\end{document}